\documentclass{article} 
\usepackage{iclr2017_conference,times}
\usepackage{hyperref}
\usepackage{url}
\usepackage{times}
\usepackage{epsfig}
\usepackage{graphicx}
\usepackage{amsmath}
\usepackage{amssymb}

\usepackage{amsfonts}
\usepackage{bm}
\usepackage{amsmath}
\usepackage{amssymb}  

\usepackage{multirow}   

\usepackage{microtype}
\usepackage{amsthm}
\newtheorem{theorem}{Theorem}
\usepackage{subfig}
\usepackage{graphicx}
\usepackage{algorithm}
\usepackage{algorithmic}
\usepackage{url}
\usepackage{comment}

\title{Privileged Multi-label Learning}

\author{Shan You$^{1}$, Chang Xu$^2$, Yunhe Wang$^{1}$, Chao Xu$^{1}$ \& Dacheng Tao$^{3}$\\
$^1$Key Lab. of Machine Perception (MOE),  Cooperative Medianet Innovation Center, \\
     ~ School of EECS, Peking University, P. R. China\\
$^2$School of Software, University of Technology Sydney\\
$^3$School of Information Technologies, University of Sydney
}

\begin{document}
\newcommand{\norm}[1]{\left\lVert#1\right\rVert}
\newcommand{\innerproduct}[2]{\left\langle#1, #2\right\rangle}

\renewcommand{\a}{\mathbf{a}}
\newcommand{\x}{\mathbf{x}}
\newcommand{\y}{\mathbf{y}}
\newcommand{\s}{\mathbf{s}}
\newcommand{\z}{\mathbf{z}}
\newcommand{\f}{\mathbf{f}}
\newcommand{\bmu}{\bm{\mu}}
\newcommand{\bsigma}{\bm{\sigma}}
\newcommand{\bTheta}{\bm{\Theta}}
\newcommand{\bSigma}{\bm{\Sigma}}
\newcommand{\w}{\bm{w}}
\newcommand{\rr}{\bm{r}}
\newcommand{\dd}{\mathbf{d}}
\newcommand{\balpha}{\bm{\alpha}}
\newcommand{\bbeta}{\bm{\beta}}
\newcommand{\g}{\mathbf{g}}
\newcommand{\e}{\mathbf{e}}
\newcommand{\bx}{\mathbf{X}}
\newcommand{\by}{\mathbf{Y}}
\newcommand{\bz}{\mathbf{Z}}
\newcommand{\be}{\mathbf{E}}
\newcommand{\bl}{\mathbf{L}}
\newcommand{\bs}{\mathbf{S}}
\newcommand{\bg}{\mathbf{G}}
\newcommand{\ba}{\mathbf{A}}
\newcommand{\bc}{\mathbf{C}}
\newcommand{\bb}{\mathbf{B}}
\newcommand{\bu}{\mathbf{U}}
\newcommand{\bq}{\mathbf{Q}}
\newcommand{\bv}{\mathbf{V}}
\newcommand{\bp}{\mathbf{P}}
\newcommand{\bw}{\mathbf{W}}
\newcommand{\bd}{\mathbf{D}}
\newcommand{\bi}{\mathbf{I}}
\newcommand{\bj}{\mathbf{J}}
\renewcommand{\u}{\mathbf{u}}
\renewcommand{\v}{\mathbf{v}}
\newcommand{\p}{\mathbf{p}}
\renewcommand{\b}{\mathbf{b}}
\newcommand{\bflambda}{\mathbf{\lambda}}
\newcommand{\A}{\mathcal{A}}
\newcommand{\X}{\mathbf{X}}
\newcommand{\W}{\mathbf{W}}
\newcommand{\sgn}{\mbox{sgn}}
\newcommand{\diag}{\mbox{diag}}
\newcommand{\armin}{\mbox{argmin}}
\newcommand{\rank}{\mbox{rank}}
\newcommand{\<}{\left\langle}
\renewcommand{\>}{\right\rangle}
\newcommand{\lbar}{\left\|}
\newcommand{\rbar}{\right\|}
\newcommand{\fan}[1]{\Vert #1 \Vert}
\newcommand{\qileft}{[\kern-0.15em[}
\newcommand{\qiLeft}{\left[\kern-0.4em\left[}
\newcommand{\qiright}{]\kern-0.15em]}
\newcommand{\qiRight}{\right]\kern-0.4em\right]}
\renewcommand{\algorithmicrequire}{\textbf{Input:}} 
\renewcommand{\algorithmicensure}{\textbf{Output:}} 
\newcommand{\eg}{\emph{e.g. }}
\newcommand{\ie}{\emph{i.e.}}
\newcommand{\wrt}{\emph{w.r.t. }}
\newcommand{\etal}{\emph{et.al. }}
\newcommand{\tw}{\tilde{\mathbf{w}}}
\newcommand{\tb}{\tilde{b}}
\newcommand{\tW}{\tilde{\mathbf{W}}}
\newcommand{\tB}{\tilde{\mathbf{b}}}

\maketitle

\begin{abstract}
	This paper presents privileged multi-label learning (PrML) to explore and exploit the relationship between labels in multi-label learning problems. We suggest that for each individual label, it cannot only be implicitly connected with other labels via the low-rank constraint over label predictors, but also its performance on examples can receive the explicit comments from other labels together acting as an \emph{Oracle teacher}. We generate privileged label feature for each example and its individual label, and then integrate it into the framework of low-rank based multi-label learning. The proposed algorithm can therefore comprehensively explore and exploit label relationships by inheriting all the merits of privileged information and low-rank constraints. We show that PrML can be efficiently solved by dual coordinate descent algorithm using iterative optimization strategy with cheap updates. Experiments on benchmark datasets show that through privileged label features, the performance can be significantly improved and PrML is superior to several competing methods in most cases.
\end{abstract}

\section{Introduction}
Different from single-label classification, multi-label learning (MLL) allows each example to own multiple and non-exclusive labels. For instance, when to post a photo taken in the scene of Rio Olympics on Instagram, Twitter or Facebook, we may simultaneously include hashtags as \#RioOlympics, \#athletes, \#medals and \#flags. Or a related news article can be simultaneously annotated as ``Sports", ``Politics" and ``Brazil". Multi-label learning aims to accurately allocate a group of labels to unseen examples with the knowledge harvested from the training data, and it has been widely-used in many applications, such as document categorization \cite{yang2009effective,li2015supervised}, image/videos classification/annotation \cite{Yang_2016_CVPR,Wang_2016_CVPR,bappy2016online}, gene function classification \cite{cesa2012synergy} and image retrieval \cite{Ranjan_2015_ICCV}.

 The most straightforward approach is 1-vs-all or Binary Relevance (BR) \cite{tsoumakas2010mining}, which decomposes the multi-label learning into a set of independent binary classification tasks. However, due to neglecting label relationships, only passable performance can be achieved. A number of methods have thus been developed for further improving the performance by taking label relationships into consideration, such as label ranking \cite{furnkranz2008multilabel},
  chains of binary classification \cite{read2011classifier}, ensemble of multi-class classification \cite{tsoumakas2011random} and label-specific features \cite{zhang2015lift}. Recently, embedding-based methods have emerged as a mainstream solution of the multi-label learning problem. The approaches assume that the label matrix is low-rank, and adopt different manipulations to embed the original label vectors, such as compressed sensing \cite{hsu2009multi}, principal component analysis \cite{tai2012multilabel}, canonical correlation analysis \cite{zhang2011multi}, landmark selection \cite{balasubramanian2012landmark} and manifold deduction \cite{bhatia2015sparse,ML2}.

 Most of low-rank based multi-label learning algorithms exploit label relationships in the hypothesis space. The hypotheses of different labels are interacted with each other under the low-rank constraint, which is as an implicit use of label relationships. By contrast, multiple labels can help each other in a more explicit way, where the hypothesis of a label is not only evaluated by the label itself, but also can be assessed by the other labels. More specifically in multi-label learning, for the label hypothesis at hand, the other labels can together act as an \textit{Oracle teacher} to provide some \textit{comments} on its performance, which is then beneficial for updating the learner. Multiple labels of examples can only be accessed in the training stage instead of the testing stage, and then Oracle teachers only exist in the training stage. This \textit{privileged} setting has been studied in LUPI  (learning using privileged information) paradigm \cite{vapnik2009learning,vapnik2009new,vapnik2015learning} and it has been reported that appropriate privileged information can boost the performance in ranking \cite{sharmanska2013learning}, metric learning \cite{fouad2013incorporating}, classification \cite{pechyony2010theory} and visual recognition \cite{Motiian_2016_CVPR}.

In this paper, we bridge connections between labels through privileged label information and then formulate an effective privileged multi-label learning (PrML) method. For each label, each example's privileged label feature can be generated from other labels. Then it is able to provide additional guidances on the learning of this label, given the underlying connections between labels. By integrating the privileged information into the low-rank based multi-label learning, each label predictor learned from the resulting model not only interacts with other labels via their predictors, but also receives explicit comments from these labels. Iterative optimization strategy is employed to solve PrML, and we theoretically show that each subproblem can be solved by dual coordinate descent algorithm with the guarantee of solution's uniqueness. Experimental results demonstrate the significance of exploiting the privileged label features and the effectiveness of the proposed algorithm.

\section{Problem Formulation}
In this section we elaborate the intrinsic privileged information in multi-label learning and formulate the corresponding privileged multi-label learning (PrML) as well.


We first introduce multi-label learning (MLL) problem and its frequent notations. Given $n$ training points, we denote the whole data set as $\mathcal{D}=  \{(\x_1,\y_1),...,(\x_n,\y_n)\}$, where $\x_i\in\mathcal{X}\subseteq\mathbb{R}^d$ is the input feature vector and $\y_i\in\mathcal{Y}\subseteq\{-1,1\}^L$ is the corresponding label vector with the label size $L$. Let $X=[\x_1,\x_2,...,\x_n]\in\mathbb{R}^{d\times n}$ be the data matrix and $Y=[\y_1,\y_2,...,\y_n]\in\{-1,1\}^{L\times n}$ be the label matrix. Specifically, $Y_{ij}=1$ if and only if the $i$-th label is assigned to the example $\x_j$ and $Y_{ij} = -1$ otherwise. Given the dataset $\mathcal{D}$, multi-label learning is formulated as learning a mapping function $f: \mathbb{R}^d \rightarrow \{-1,1\}^L$ that can accurately predict labels for unseen test points.

\subsection{Low-rank multi-label embedding}
A straightforward manner to parameterize the decision function is using linear classifiers, \ie~$f(\x) = Z^T\x = [\z_1,...,\z_L]^T\x$ where $Z\in\mathbb{R}^{d\times L}$. Note that the linear form is actually incorporated with the bias term by augmenting an additional 1 to the feature vector $\x$. Binary Relevance (BR) method \cite{tsoumakas2010mining} decomposes
multi-label learning into a set of single-label learning problems. The binary classifier for each label can be obtained by the  widely-used SVM method:
\begin{equation}\label{BR}
\begin{split}
\displaystyle{\min_{\z_i=[\z^*_i;b_i],\bm{\xi}_{i}}} &\quad \displaystyle{\dfrac{1}{2}\norm{\z_i^*}_2^2  + C\sum_{j=1}^n} \xi_{ij}\\
\mbox{s.t.} &\quad Y_{ij}(\langle \z_i^*,\x_j\rangle+b_i)\geq 1- \xi_{ij} \\
&\quad \xi_{ij} \geq 0, \forall  j=1,...,n,
\end{split}
\end{equation}
where $\bm{\xi}_i = [\xi_{i1},...,\xi_{in}]^T$ is slack variable and $\langle\cdot\rangle$ is the inner product between two vectors or matrices. Predictors $\{\z_1,...,\z_L\}$ of different labels are thus independently solved without considering relationships between labels, which limits the classification performance of BR method.

Some labels can be closely connected and used to occur together on examples, and thus the label matrix is often supposed to be low-rank, which leads to the low rank of label predictor matrix $Z=[\z_1,...,\z_L]$ as a result. Considering the rank of $Z$ as $k$, which is smaller than $d$ and $L$, we are able to employ two smaller matrices to approximate $Z$, \ie~$Z=D^TW$. $D\in\mathbb{R}^{k\times d}$ can be seen as a dictionary of hypotheses in latent space $\mathbb{R}^k$, while each $\w_i$ in
$W=[\w_1,...,\w_L]\in\mathbb{R}^{k\times L}$ is the coefficient vector to generate the predictor of $i$-th label from the hypothesis dictionary $D$. Each classifier $\z_i$ is represented as $\z_i = D^T \w_i~(i =1,2,...,L)$ and Problem \eqref{BR} can be extended into:
\begin{equation}\label{lowrank}
\begin{split}
\min_{D,W,\xi} &\quad \dfrac{1}{2}(\norm{D}_F^2 + \sum_{i=1}^L\norm{\w_i}_2^2)  + C\sum_{i=1}^L\sum_{j=1}^n \xi_{ij}\\
\mbox{s.t.} &\quad Y_{ij}(\langle D^T\w_i,\x_j\rangle)\geq 1- \xi_{ij} \\
&\quad \xi_{ij} \geq 0, \forall  i=1,...,L; j=1,...,n,
\end{split}
\end{equation}
where $\xi=[\bm{\xi}_1,...,\bm{\xi}_L]^T$. Thus in Eq.\eqref{lowrank}, the classifiers of all labels $\z_i$ are drawn from an identical low-dimensional subspace, \ie~the row space of $D$. Then using block coordinate descent, either $D$ or $W$ can be solved within the empirical risk minimization (ERM) framework by turning it into a hinge loss minimization problem.



\subsection{Privileged information in multi-label learning}
The slack variable $\xi_{ij}$ in Eq.\eqref{lowrank} indicates the prediction error of the $j$-th example on the $i$-th label. In fact, it depicts the error-tolerant ability of a model, and is directly related to the optimal classifier and its classification performance. From a different point of view, slack variables can be regarded as \emph{comments} of some \emph{Oracle Teacher} on the performance of predictors on each example. In multi-label context for each label, its hypothesis is not only evaluated by itself, but also assessed by the other labels. Thus other labels can be seen as its Oracle teacher, who will provide some comments during this label's learning. Note that these label values are known as a priori only during training; when we get down to learning the $i$-th label's predictor, we actually know the values of other labels for each training point $\x_j$. Therefore, we can formulate the other label values as privileged information (or hidden information) of each example. Let
\begin{equation}\label{pri}
\tilde{\y}_{i,j} \stackrel{\vartriangle}{=} \y_j,~ \mbox{with} ~i\mbox{-th element being 0}.
\end{equation}
We call $\tilde{\y}_{i,j}$ the training point $\x_j$'s \textit{privileged label feature} on the $i$-th label. It can be seen that the privileged label space is constructed straightforwardly from the original label space. These privileged label features can thus be regarded as an explicit way to connect all labels. In addition, note that the \textit{valid} dimension (removing 0) of $\tilde{\y}_{i,j}$ is $L-1$, since we take the other $L-1$ label values as the privileged label features. Moreover, not all the other labels have the positive impact on the learning of some label \cite{sun2014multi}, and thus it is appropriate to strategically select some key labels to formulate the privileged label features. We will discuss this in the Experiment section.

Since for each label, the other labels serve as the Oracle teacher via the privileged label feature $\tilde{\y}_{i,j}$ on each example, the comments on slack variables can be modelled as a linear function \cite{vapnik2009new},
\begin{equation}\label{slack}
\xi_{ij}(\tilde{\y}_{i,j};\tilde{\w}_i) = \langle\tilde{\w}_i,\tilde{\y}_{i,j}\rangle.
\end{equation}
The function $\xi_{ij}(\tilde{\y}_{i,j};\tilde{\w}_i)$ is thus called \textit{correcting function} with respect to the $i$-th label, where $\tilde{\w}_i$ is the parameter vector. As shown in Eq.\eqref{slack}, the privileged comments $\tilde{\y}_{i,j}$ directly correct the values of slack variables as the prior knowledge or the additional information. Integrating privileged features as Eq.\eqref{slack} into the SVM stimulates the popular SVM+ method \cite{vapnik2009new}, which has been proved to improve the convergence rate and the performance.



Integrating the proposed privileged label features into the low-rank parameter structure as Eqs.\eqref{lowrank} and \eqref{slack}, we formulate a new multi-label learning model, privileged multi-label learning (PrML) by casting it into the  SVM+-based LUPI paradigm,
\begin{equation}\label{PrML}
\begin{array}{rl}
\displaystyle{\min_{D,W,\tilde{W}}} & \displaystyle{\dfrac{1}{2}\norm{D}_F^2 + \dfrac{1}{2}\sum_{i=1}^L(\gamma_1\norm{\w_i}_2^2 + \gamma_2\norm{\tilde{\w}_i}_2^2) }
	\displaystyle{+~C\sum_{i=1}^L\sum_{j=1}^n \langle \tilde{\w}_i, \tilde{\y}_{i,j}\rangle}\\
\mbox{s.t.} & Y_{ij}\langle D^T\w_i,\x_j\rangle\geq 1-  \langle \tilde{\w}_i, \tilde{\y}_{i,j}\rangle\\
& \langle \tilde{\w}_i, \tilde{\y}_{i,j}\rangle\geq 0, \forall i=1,...,L; j=1,...,n,
\end{array}
\end{equation}
where $\tilde{W}=[\tilde{\w}_1,...,\tilde{\w}_L]$. Particularly, we absorb the bias term to obtain a compact variant of the original SVM+, because it is turned out to have a simpler form in the dual space and can be solved more efficiently. In this way, the training data within multi-label learning is actually in the triplet fashion, \ie~$(\x_i,\y_i,\tilde{Y}_i), i=1,...,n$, where $\tilde{Y}_i = [\tilde{\y}_{1,i},...,\tilde{\y}_{L,i}]$ is the privileged label feature matrix for each label.


\textbf{Remark.} When $W = I$, \ie~the low-dimensional projection is identical, the proposed PrML degenerates into a simpler BR-style model (we call it privileged Binary Relevance, PrBR), where the whole model decomposes into $L$ independent binary models. However, every single model is still combined with the comments form privileged information, thus it may still be superior to BR.


\section{Optimization}
In this section, we present how to solve the proposed privileged multi-label learning algorithm Eq.(\ref{PrML}). The whole model of Eq.(\ref{PrML}) is not convex due to the multiplication of $D^T\w_i$ in constraints. However, each subproblem with fixed $D$ or $W$ is convex, thus it can be solved by various efficient convex solvers. Note that $\langle D^T\w_i,\x_j\rangle$ has two equivalent forms, \ie~$\langle \w_i,D\x_j\rangle$ and $\langle D, \w_i\x_j^T\rangle$, and thus the correcting function can be coupled with $D$ or $W$, without damaging the convexity of either subproblem. In this way, Eq.(\ref{PrML}) can be solved using the alternative iteration strategy, \ie~iteratively conducting the following two steps: optimizing $W$ and privileged variable $\tilde{W}$ with fixed $D$, and updating $D$ and privileged variable $\tilde{W}$ with fixed $W$. Both subproblems are related to SVM+, inducing their dual problems to be quadratic programming (QP). In the following, we elaborate the solving process in real implementations.

\subsection{Optimizing $W, \tilde{W}$ with fixed $D$}
Fixing $D$, Eq.(\ref{PrML}) can be decomposed into $L$ independent binary classification problems, each of which regards the variable pair $(\w_i,  \tilde{\w}_i)$. Parallel techniques or multi-core computation can thus be employed to speed up the training process. In specific, the optimization problem with respect to $(\w_i,  \tilde{\w}_i)$ is
\begin{equation}\label{w_i}
\begin{split}
\min_{\w_i,\tilde{\w}_i} & ~\dfrac{1}{2}(\gamma_1\norm{\w_i}_2^2 + \gamma_2\norm{\tilde{\w}_i}_2^2)  + C\sum_{j=1}^n \langle \tilde{\w}_i, \tilde{\y}_{i,j}\rangle\\
\mbox{s.t.} & ~Y_{ij}\langle \w_i,D\x_j\rangle\geq 1- \langle \tilde{\w}_i, \tilde{\y}_{i,j}\rangle \\
& \langle \tilde{\w}_i, \tilde{\y}_{i,j}\rangle \geq 0, \forall  j=1,...,n.
\end{split}
\end{equation}
and its dual form is (see supplementary materials)
\begin{equation}\label{dual_w_i}
\begin{split}
\max_{\balpha,\bbeta}~& -\frac{1}{2} (\balpha\circ\y_i^*)^TK_D(\balpha\circ\y_i^*)+\bm{1}^T\balpha  - \frac{1}{2\gamma}(\balpha+\bbeta-C\bm{1})^T\tilde{K}_i (\balpha+\bbeta-C\bm{1})
\end{split}
\end{equation}
with the parameter update $\gamma \leftarrow \gamma_2/\gamma_1, C\leftarrow C/\gamma_1$ and
the constraints $\balpha\succeq0,\bbeta\succeq0$, \ie~$\alpha_j\geq0,\beta_j\geq0,\forall j\in [1:n].$ Moreover, $\y_i^* = [Y_{i1},Y_{i2},...,Y_{in}]^T$ is the label-wise vectors for the $i$-th label. $\circ$ is the Hadamard (element-wise) product of two vectors or matrices. $K_D\in\mathbb{R}^{n\times n}$ is the $D$-based features' inner product (kernel) matrix with $K_D(j,q) = \langle D\x_j,D\x_q\rangle$. $\tilde{K}_i $ is the privileged label features' inner product (kernel) matrix with respect to the $i$-th label, where $\tilde{K}_i(j,q) = \langle\tilde{\y}_{i,j}, \tilde{\y}_{i,q} \rangle$. $\bm{1}$ is the vector with all ones.

\cite{pechyony2010smo} proposed an SMO-style algorithm (gSMO) for SVM+ problem. However, because of the bias term, the Lagrange multipliers  are tangled together in the dual problem, which leads to a more complicated constraint set
$$\{(\balpha,\bbeta)|\balpha^T\y_i^* = 0, \bm{1}^T(\balpha+\bbeta-C\bm{1}) = 0,\balpha\succeq0, \bbeta\succeq0\}$$
than $\{(\balpha,\bbeta)|\balpha\succeq0, \bbeta\succeq0\}$ in our PrML. Hence by absorbing the bias term, Eq.\eqref{w_i} can produce a more compact dual problem only with non-negative constraint. Coordinate descent (CD) \footnote{We optimize an equivalent ``min" problem instead of the original ``max" one.} algorithm can be applied to solve the dual problem, and a closed-form solution can be obtained in each iteration step \cite{li2016fast}. After solving the Eq.(\ref{dual_w_i}), according to the Karush-Kuhn-Tucker (KKT) conditions, the optimal solution for the primal problem (\ref{w_i}) can be expressed by the Lagrange multipliers:
\begin{align}
\w_i &= \sum_{j=1}^n \alpha_jY_{ij}D\x_j \label{wi1}\\
\tilde{\w}_i &= \frac{1}{\gamma}\sum_{j=1}^n (\alpha_j+\beta_j-C)\tilde{\y}_{i,j} \label{w_i_expression}
\end{align}

\subsection{Optimizing $D, \tilde{W}$ with fixed $W$}
Given fixed coefficient matrix $W$, we update and learn the linear transformation $D$ with the help of comments provided by privileged information. Thus the problem (\ref{PrML}) for $(D,\tilde{W})$ is reduced to
\begin{equation}\label{D}
\begin{split}
\displaystyle{\min_{D,\tilde{W}}}~ & \displaystyle{\dfrac{1}{2}\norm{D}_F^2 + \dfrac{\gamma_2}{2}\sum_{i=1}^L\norm{\tilde{\w}_i}_2^2  + C\sum_{i=1}^L\sum_{j=1}^n} \langle \tilde{\w}_i, \tilde{\y}_{i,j}\rangle\\
\mbox{s.t.} ~& Y_{ij}\langle D, \w_i\x_j^T\rangle\geq 1- \langle \tilde{\w}_i, \tilde{\y}_{i,j}\rangle \\
& \langle \tilde{\w}_i, \tilde{\y}_{i,j}\rangle \geq 0, \forall i=1,...,L;  j=1,...,n.
\end{split}
\end{equation}
Eq.(\ref{D}) has $Ln$ constraints, each of which can be indexed with a two-dimensional subscript $[i,j]$. The Lagrange multipliers of Eq.(\ref{D}) are thus two-dimensional as well. To make the dual problem of Eq.(\ref{D}) consistent with Eq.(\ref{dual_w_i}), we define a bijection $\phi: [1:L]\times[1:n] \rightarrow [1:Ln]$ as the row-based vectorization index mapping, \ie~$\phi([i,j]) = (i-1)n+j$. In a nutshell, we arrange the constraints (also the multipliers) according to the order of row-based vectorization. In this way, the corresponding dual problem of Eq.(\ref{D}) is formulated as (see details in supplementary materials)
\begin{equation}\label{dual_D}
\begin{split}
\max_{\balpha\succeq0,\bbeta\succeq0}~& -\frac{1}{2} (\balpha\circ\y^*)^TK_W(\balpha\circ\y^*)+\bm{1}^T\balpha  - \frac{1}{2\gamma_2}(\balpha+\bbeta-C\bm{1})^T\tilde{K}(\balpha+\bbeta-C\bm{1})
\end{split}
\end{equation}
where $\y^* = [\y^*_1;\y^*_2;...;\y^*_L]$ is the row-based vectorization of $Y$ and $\tilde{K}=diag(\tilde{K}_1,\tilde{K}_2,...,\tilde{K}_L)$ is a block diagonal matrix, which corresponds to the kernel matrix of privileged label features. $K_W$ is the kernel matrix of input features with every element $K_W(s,t) =\langle G_{\phi^{-1}(s)}, G_{\phi^{-1}(t)}\rangle$, where $G_{ij} = \w_i\x_j^T$. Based on the KKT conditions, $(D,\tilde{W})$ can be constructed using $(\balpha,\bbeta)$:
\begin{align}
D &= \sum_{s=1}^{Ln}\alpha_sy^*_sG_{\phi^{-1}(s)}\label{D1}\\
\tilde{\w}_i & = \frac{1}{\gamma_2}\sum_{j=1}^n (\alpha_{\phi([i,j])} + \beta_{\phi([i,j])} - C) \tilde{\y}_{i,j} \label{w_i_expression_2}
\end{align}
In this way, Eq.(\ref{dual_D}) has an identical optimization form with Eq.(\ref{dual_w_i}). Thus we can also turn it to the fast CD method \cite{li2016fast}. However, due to the script index mapping, directly using the method proposed in \cite{li2016fast} is very expensive. Considering the privileged kernel matrix $\tilde{K}$ is block sparse, we can further speed up the calculation. Details of the modified version of dual CD algorithm for solving Eq.\eqref{dual_D} are presented in Algorithm \ref{alg:D}. Also note that one primary merit of this algorithm is the free calculation of the whole kernel matrix. Instead, we only need to calculate its diagonal elements as line 2 in Algorithm \ref{alg:D}.

\begin{algorithm}[tb]
	\caption{A dual coordinate descent algorithm for solving Eq.(11)}
	\label{alg:D}
	\begin{algorithmic}[1]
		\REQUIRE{Training data: feature matrix $X=[\x_1,\x_2,...,\x_n]\in\mathbb{R}^{d\times n}$, label matrix $Y=[\y_1,\y_2,...,\y_n]\in\{-1,1\}^{L\times n}$. Privileged label features: $\tilde{\y}_{i,j}$ of $i$-th label and $j$-th example. Low-dimensional embedding projection: $W=[\w_1,...,\w_L]\in\mathbb{R}^{k\times L}$. Learning parameters: $\gamma,C>0$.}
		\STATE Define the bijection $\phi: [1:L]\times[1:n] \rightarrow [1:Ln]$ as the row-based vectorization index mapping, \ie~$\phi([i,j]) = (i-1)n+j$. Denote its inverse as $\phi^{-1}$.
		\STATE Constuct $Q\in\mathbb{R}^{2Ln}$ for each $s\in[1:2Ln]$: for $1\leq s\leq Ln$, $[i,j]\leftarrow\phi^{-1}(s)$, $Q_s = \norm{\w_i}_2^2\norm{\x_j}^2_2 + 1/\gamma\norm{\tilde{\y}_{i,j}}_2^2$; for $Ln+1\leq s\leq 2Ln$, $[i,j]\leftarrow\phi^{-1}(s-Ln)$, $Q_s = 1/\gamma\norm{\tilde{\y}_{i,j}}_2^2$.
		\STATE Initialization: $\bm{\eta}=[\balpha;\bbeta]\leftarrow \bm{0}$, $D\leftarrow \bm{0}$ and $\tilde{\w}_i  \leftarrow -\frac{C}{\gamma}\sum_{j=1}^n\tilde{\y}_{i,j}$ for each $i\in[1:L]$.
		\WHILE{not convergence}
		\STATE randomly pick an index $s$
		\IF{$1\leq s\leq Ln$}
		\STATE $[i,j]\leftarrow\phi^{-1}(s)$
		\STATE $\nabla_s \leftarrow Y_{ij}\w_i^TD\x_j - 1 + \tilde{\w}_i^T\tilde{\y}_{i,j}$
		\STATE $\delta \leftarrow \max\{-\eta_s,-\nabla_s/Q_s\}$
		\STATE $D\leftarrow D + \delta Y_{ij}\w_i\x_j^T$
		\ELSIF{$Ln+1\leq s\leq 2Ln$}
		\STATE $[i,j]\leftarrow\phi^{-1}(s-Ln)$
		\STATE $\nabla_s \leftarrow \tilde{\w}_i^T\tilde{\y}_{i,j}$
		\STATE $\delta \leftarrow \max\{-\eta_s,-\nabla_s/Q_s\}$
		\ENDIF
		\STATE $\tilde{\w}_i \leftarrow \tilde{\w}_i + (\delta/\gamma)\tilde{\y}_{i,j}$
		\STATE $\eta_s \leftarrow \eta_s + \delta$
		\ENDWHILE
		\ENSURE{Dictionary $D$ and correcting functions $\{\tilde{\w}_i\}_{i=1}^L$.}
	\end{algorithmic}
\end{algorithm}

\begin{algorithm}[tb]
	\caption{Privileged Multi-label Learning (PrML)}
	\label{alg}
	\begin{algorithmic}[1]
		\REQUIRE{Training data: feature matrix $X=[\x_1,\x_2,...,\x_n]\in\mathbb{R}^{d\times n}$, label matrix $Y=[\y_1,\y_2,...,\y_n]\in\{-1,1\}^{L\times n}$. Learning parameters: $\gamma_1,\gamma_2,C\geq0$.}
		\STATE Construction of privileged label features for each label and each training point, \eg as Eq.(\ref{pri}).
		\STATE initialization of $D$
		\WHILE{not convergence}
		\FOR{each $i\in[1:L]$}
		\STATE  $[\balpha,\bbeta] \leftarrow $ solving Eq.(\ref{dual_w_i})
		\STATE update $\w_i,\tilde{\w}_i$ according to Eq.(\ref{wi1}) and Eq.(\ref{w_i_expression})
		\ENDFOR
		\STATE $[\balpha,\bbeta] \leftarrow $ solving Eq.(\ref{dual_D})
		\STATE update $D,\tilde{W}$ according to Eq.(\ref{D1}) and Eq.(\ref{w_i_expression_2})
		\ENDWHILE
		\ENSURE{A linear multi-label classifier $Z = D^TW$, together with a correcting function $\tilde{W}$ \wrt $L$ labels.}
	\end{algorithmic}
\end{algorithm}

\subsection{Framework of PrML}
Our proposed privileged multi-label learning is summarized in Algorithm \ref{alg}. As indicated in Algorithm \ref{alg}, both $D$ and $W$ are updated with the help of comments from privileged information. Note that the primal variables and dual variables are connected with KKT connections, and thus in real applications lines 5-6 and 8-9 in Algorithm \ref{alg} can be implemented iteratively. Since each subproblem is actually a linear SVM+ optimization and solved by the CD method, its convergence is consistent with that of the dual CD algorithm for linear SVM \cite{hsieh2008dual}. Due to the cheap updates, \cite{hsieh2008dual,li2016fast} empirically showed it can be much faster than GMO-style methods and many other convex solvers when $d$ (number of features) is large. Moreover, the independence of labels in Problem \eqref{w_i} enables to use parallel techniques and multicore computation to accommodate the large $L$ (number of labels). As for a large $n$ (number of examples) (also large $L$ for Problem \eqref{D}), we can use mini-batch CD method \cite{takac2015distributed}
, where each time a batch of examples are selected and CD updates are parallelly applied to them, \ie~lines 5-17 can be implemented parallelly. Also recently \cite{kdd2016} designed a framework for parallel CD and achieved significant speeding up even when the $d$ and $n$ are very large. Thus, our model can scale to $d$, $L$ and $n$. In addition, the solution for each of subproblem is also unique, as Theorem \ref{theorem1} stated.
\begin{theorem}\label{theorem1}
	The solution to the problem (\ref{w_i}) or (\ref{D}) is unique for any $\gamma_1>0,\gamma_2>0,C>0$.
\end{theorem}
\begin{proof}[Proof skeleton]
	Both Eq.(6) and Eq.(10) can be cast into an identical SVM+ optimization with the form of objective function being $F = \frac{1}{2}\norm{\w}_2^2 + \frac{\gamma}{2}\norm{\tilde{\w}}_2^2 + C\sum_{j=1}^n \langle\tilde{\w},\tilde{\x}_j\rangle$ and a closed convex feasible solution set. Denote $\u: = (\w;\tilde{\w})$ and assume two optimal solutions $\u_1, \u_2$, we have $F(\u_1)=F(\u_2)$. Let $\u_t = (1-t)\u_1 + t\u_2, t\in[0,1]$, then $F(\u_t)\leq(1-t)F(\u_1)+tF(\u_2) = F(\u_1)$, thus $F(\u_t)=F(\u_1)=F(\u_2)$ for all $t\in[0,1]$, which implies that $g(t) = F(\u_t)-F(\u_1)\equiv 0, \forall t\in[0,1]$. Moreover, we have
	\begin{align*}
	g(t)=&\frac{1}{2} t^2 \norm{\w_2-\w_1}_2^2 + t\langle\w_2-\w_1,\w_1\rangle +  \frac{\gamma}{2} t^2 \norm{\tilde{\w}_2-\tilde{\w}_1}_2^2+ t\langle\tilde{\w}_2-\tilde{\w}_1,\tilde{\w}_1\rangle + tC\sum_{i=1}^n \langle \tilde{\w}_2-\tilde{\w}_1,\tilde{\x}_i\rangle\\
	g''(t) = &\norm{\w_2-\w_1}_2^2 + \gamma\norm{\tilde{\w}_2-\tilde{\w}_1}_2^2 = 0.
	\end{align*}
	For $\gamma>0$, then we have $\w_1=\w_2$ and $\tilde{\w}_1=\tilde{\w}_2$.
\end{proof}

Proof of Theorem \ref{theorem1} mainly lies in the strict convexity of the objective function in either Eq.\eqref{w_i} or \eqref{D}. Concrete details are referred to the supplementary materials. In this way, the correcting function $\tilde{W}$ serves as a bridge to channel the $D$ and $W$, and the convergence of $\tilde{W}$ infers the convergence of $D$ and $W$. Thus we can take $\tilde{W}$ as the barometer of the whole algorithm's convergence.

\begin{table}[t]
	\footnotesize
	\centering
	\caption{Data statistics. $n$ is the total number of examples. $d$ and $L$ are the number of features and labels, respectively; $\bar{L}$ and Den($L$) are the average number of positive labels in an instance and the label density, respectively. `Type' means feature type. }\label{data}
	\begin{tabular}{l|r|r|r|r|c|c}
		Dataset &n &$d$ & $L$ & $\bar{L}$ & Den($L$)&type \\ \hline
		enron &1702 &1001 & 53 & 3.378 &0.064&nominal\\
		yeast &2417 &103 & 14 & 4.237&0.303&numeric \\
		corel5k &5000 &499 & 374  &3.522&0.009&nominal\\
		bibtex & 7395&1836 & 159 &  2.402&0.015&nominal \\
		eurlex & 19348 &5000&3993&5.310& 0.001 & nominal\\
		mediamill &43907 &120 & 101 & 4.376&0.043&numeric \\	\hline
	\end{tabular}
\end{table}

\begin{figure*}[t]
	\centering
	\includegraphics[width=\linewidth]{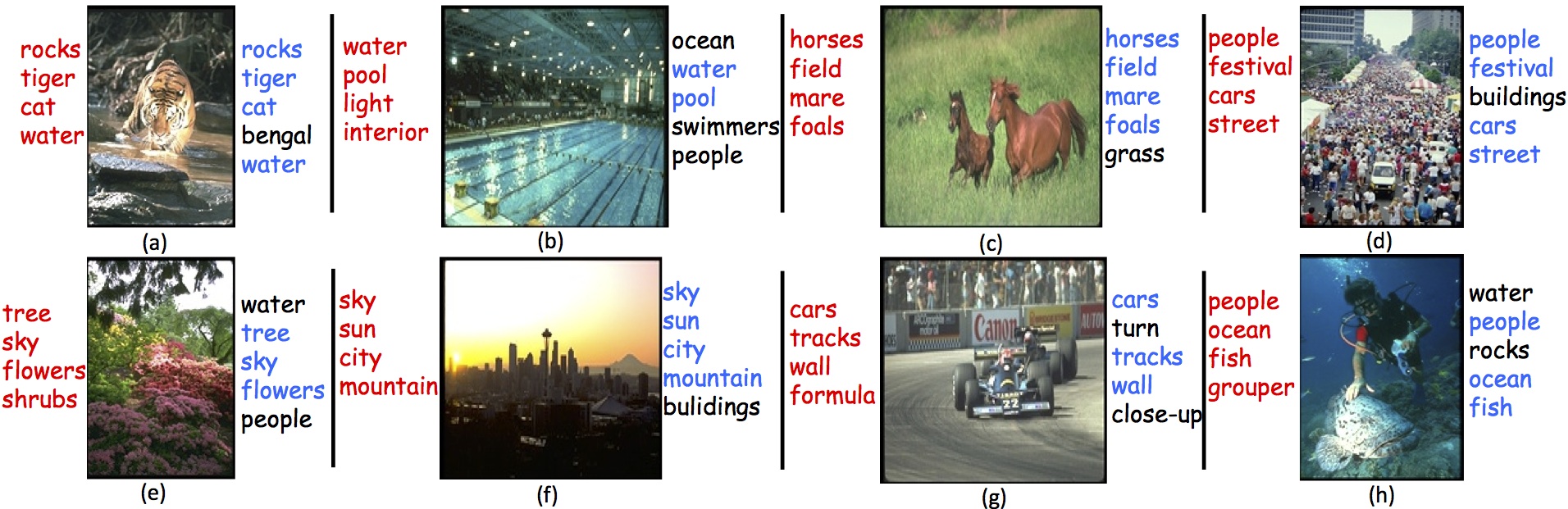}
	\caption{Some of the image annotation results of the proposed PrML on the benchmark corel5k dataset. Tags: red, ground-truth; blue, correct predictions; black, wrong predictions. }
	\label{fig:images}
\end{figure*}

\section{Experimental Results}
In this section, we conduct various experiments on benchmark datasets to validate the effectiveness of using the intrinsic privileged information for multi-label learning. In addition, we also investigate the performance and superiority of the proposed PrML model comparing to recent competing multi-label methods.

\subsection{Experiment configuration}
\textbf{Datasets.} We select six benchmark multi-label datasets, including enron, yeast, corel5k, bibtex, eurlex and mediamill. Specially, we consider the cases when $d$ (eurlex), $L$ (eurlex) and $n$ (corel5k, bibtex, eurlex \& mediamill) are large respectively. Also note that enron, corel5k, bibtex and eurlex are of sparse features. See Table \ref{data} for the details of these datasets.

\textbf{Comparison approaches.} \\
1). BR (binary relevance) \cite{tsoumakas2010mining}. A SVM is trained with respect to each label. \\
2). ECC (ensembles of classifier chains) \cite{read2011classifier}. It turns ML into a series of binary classification problems. \\
3). RAKEL (random k-labelsets) \cite{tsoumakas2011random}. It transforms ML into an ensemble of multi-class classification problems. \\
4). LEML (low rank empirical risk minimization for multi-label learning) \cite{yu2014large}. It is a low-rank embedding approach
which is casted into ERM framework.\\
5). ML$^2$ (multi-label manifold learning) \cite{ML2}. It is a latest multi-label learning method, which is based on the manifold assumption in label space.

\textbf{Evaluation Metrics.} We use six prevalent metrics to evaluate the performance of all methods, including Hamming loss, One-error, Coverage, Ranking loss, Average precision (Aver precision) and Macro-averaging AUC (Mac AUC).  Note that all evaluation metrics have the value range [0,1]. In addition, for the first four metrics, the smaller values would indicate the better classification performance and we use $\downarrow$ to index this positive logic. On the contrary, for the last two metrics larger values represent the better performance, indexed by $\uparrow$.

\begin{figure}[t]
	\centering
	\subfloat[Hamming loss $\downarrow$]
	{\includegraphics[width=0.32\columnwidth]{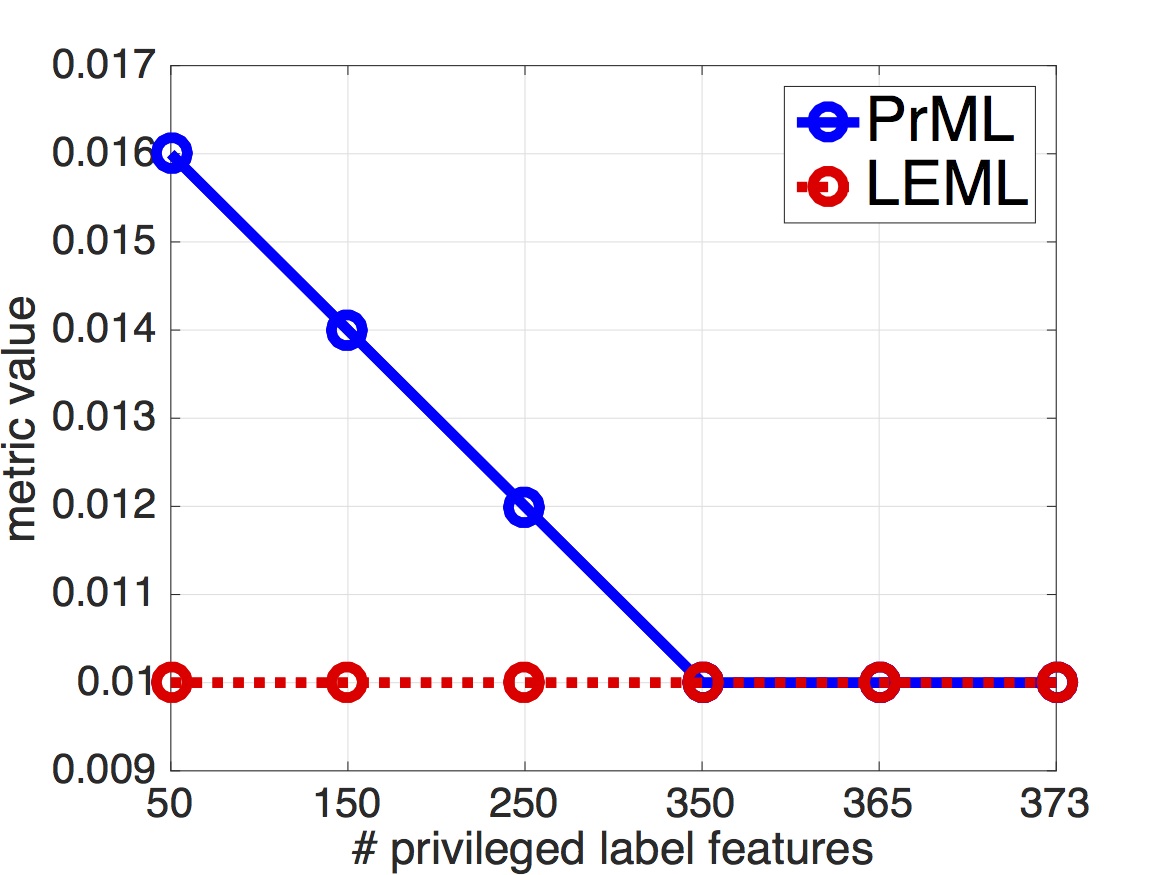}}
	\ 
	\subfloat[One-error $\downarrow$]
	{\includegraphics[width=0.32\columnwidth]{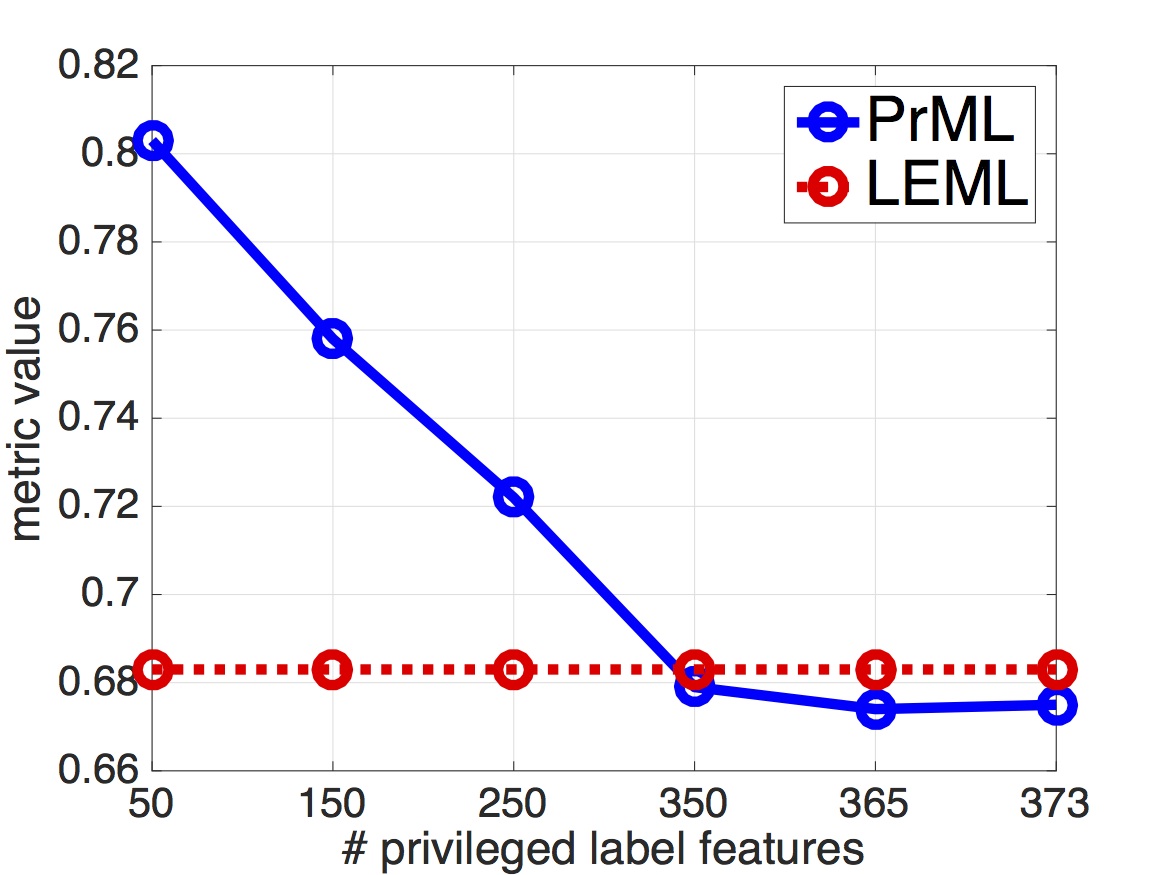}}
	\ 
	\subfloat[Coverage $\downarrow$]
	{\includegraphics[width=0.32\columnwidth]{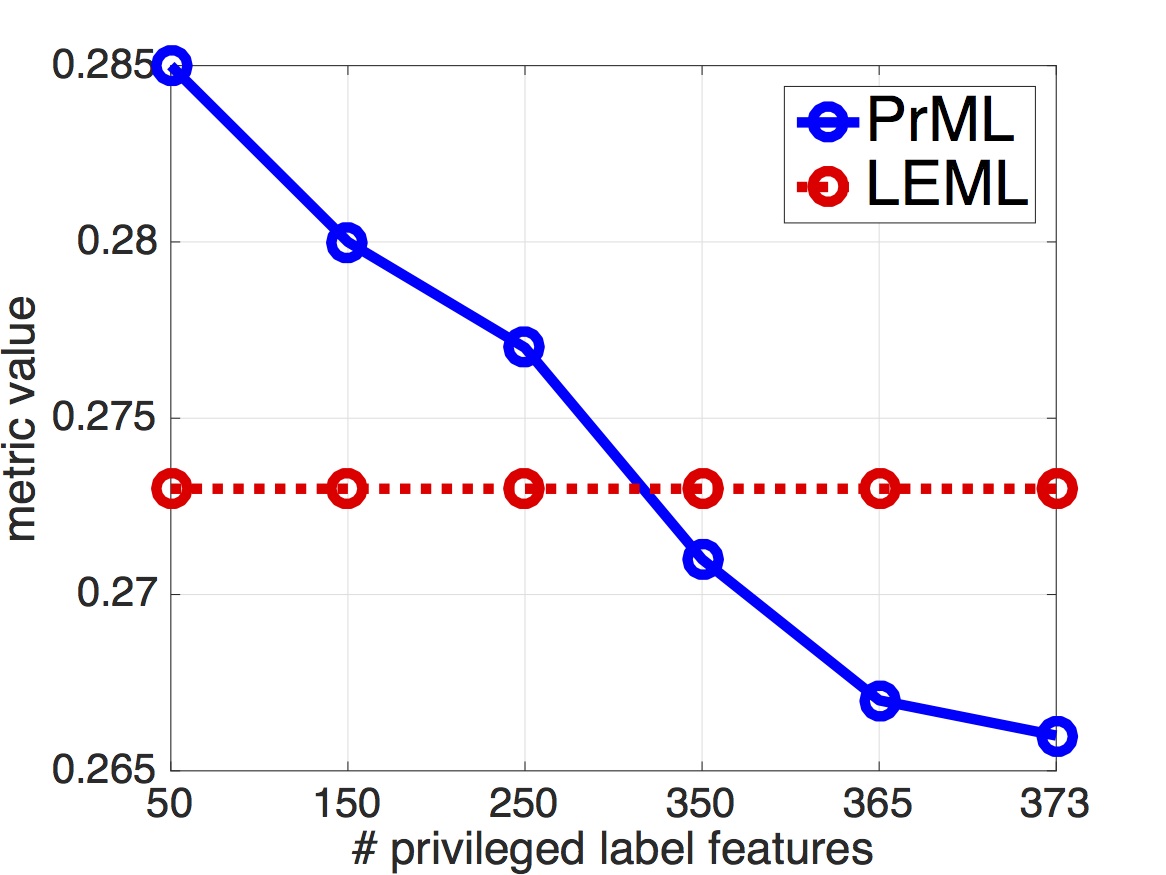}}\\
	\subfloat[Ranking loss $\downarrow$]
	{\includegraphics[width=0.32\columnwidth]{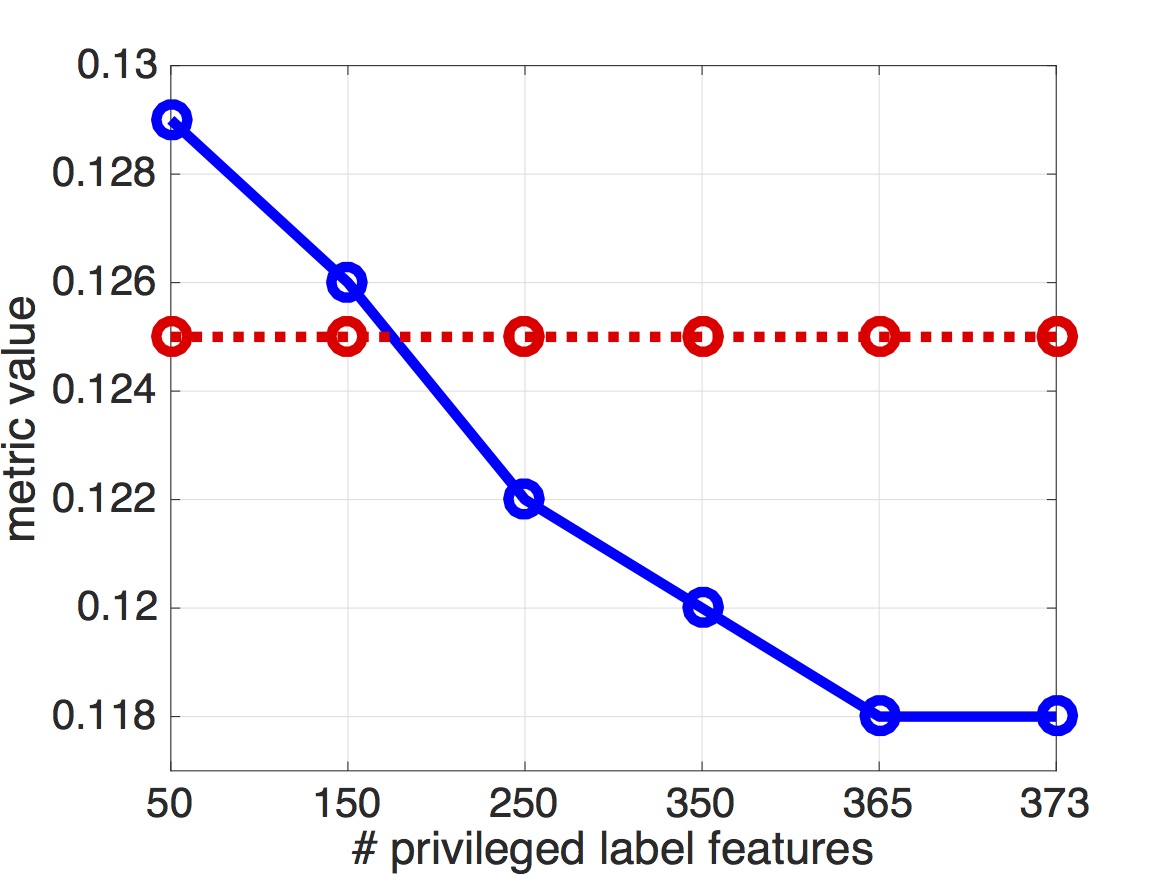}}
	\ 
	\subfloat[Average precision $\uparrow$]
	{\includegraphics[width=0.32\columnwidth]{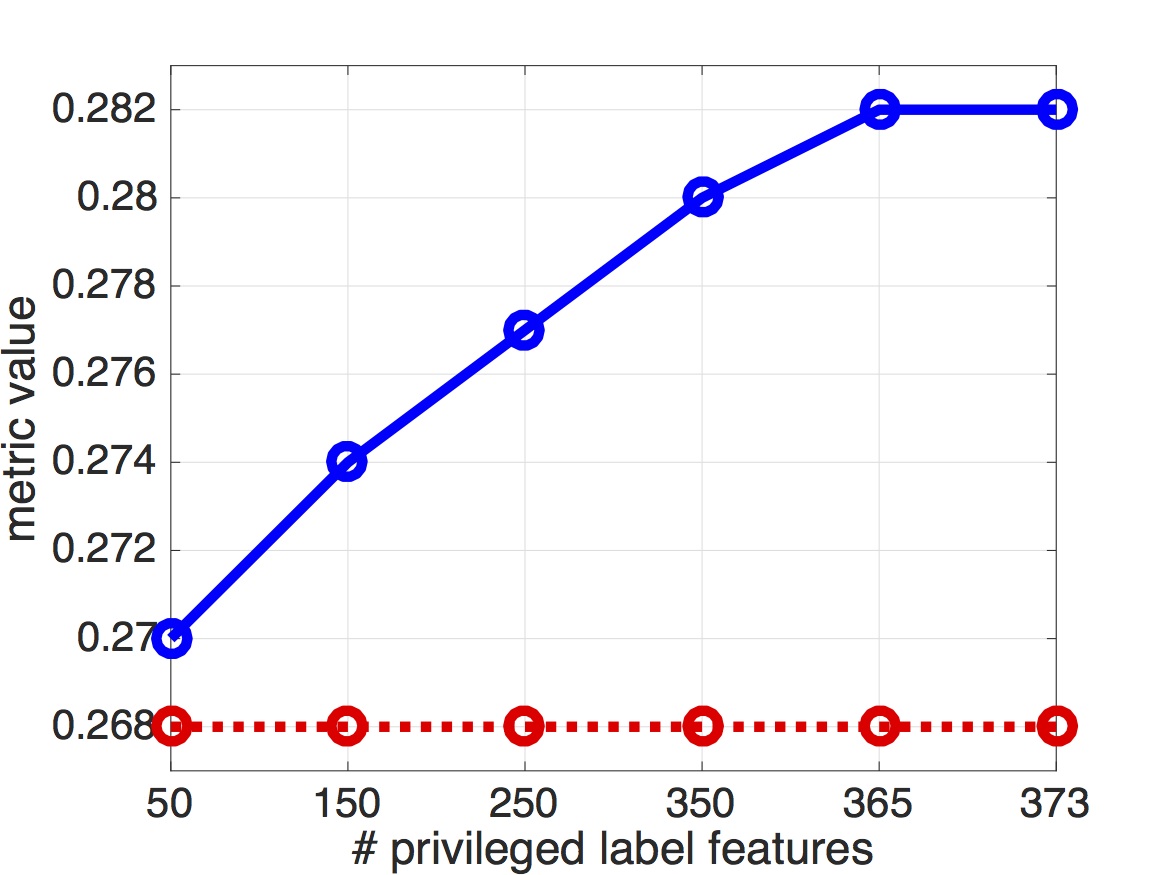}}
	\ 
	\subfloat[Macro-averaging AUC $\uparrow$]
	{\includegraphics[width=0.32\columnwidth]{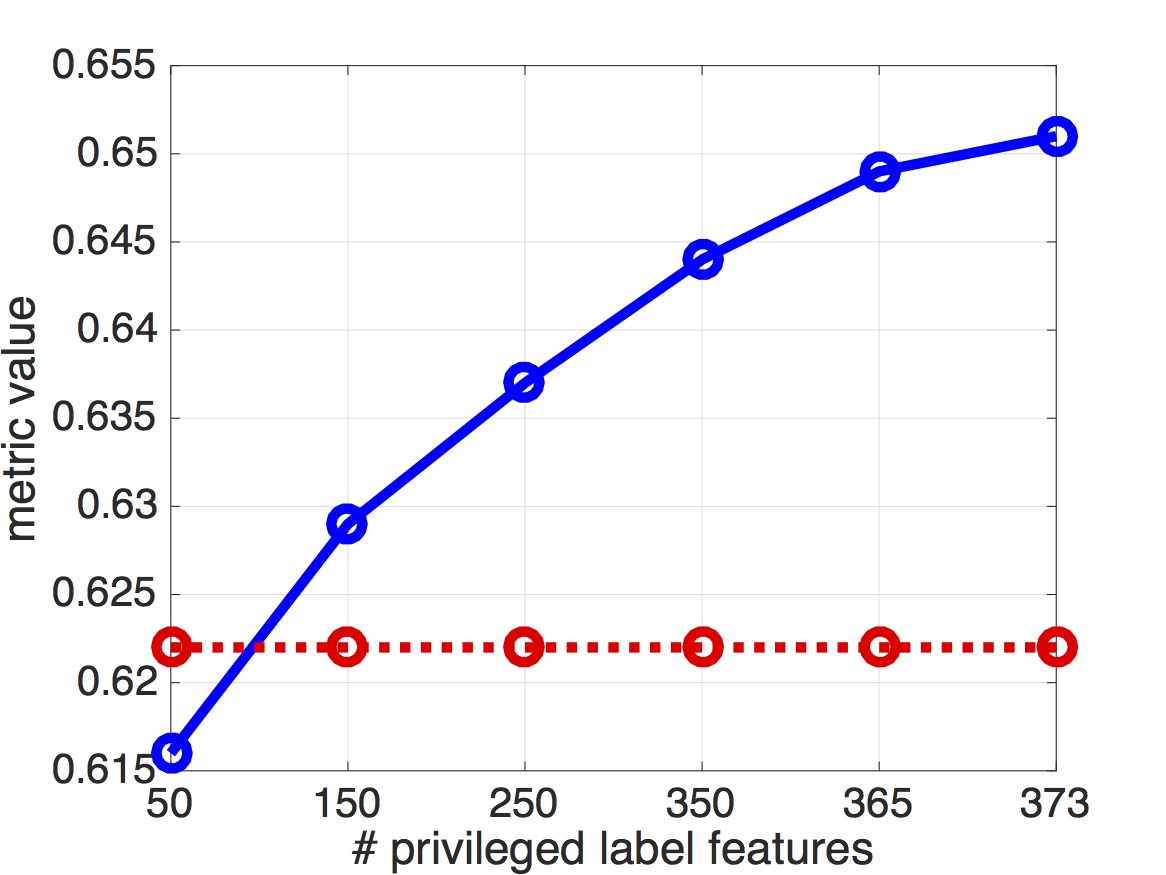}}
	\caption{Classification results of PrML (blue solid line) and LEML (red dashed line) on corel5k (50\% for training \& 50\% for testing) \wrt different dimension of privileged label features.}
	\label{fig1}
\end{figure}

\subsection{Algorithm analysis}

\textbf{Performance visualization. } First we analyze our proposed PrML and on a global sense, we select the benchmark image dataset corel5k and visualize the results of image annotation to directly examine how PrML functions.  For the learning process, we randomly selected 50\% examples without repeating as the training set and the rest ones as the testing set. In our experiment, parameter $\gamma_1$ is set to be 1; $\gamma_2$ and $C$ are in the range of $0.1\sim2.0$ and $10^{-3}\sim20$ respectively, and determined using cross validation by a part of training points.  The embedding dimension $k$ is set to be $k=\lceil0.9L\rceil$, where $\lceil r\rceil$ is the smallest integer greater than $r$. Some of the annotation results is presented in Figure \ref{fig:images}, where the left tags are the ground-truth and the right ones are the top five predicted tags.

As shown in Figure \ref{fig:images}, we can safely conclude that the proposed PrML performs well on the image annotation tasks. It can predict correctly the semantic labels in most cases. Note that although in some cases the predicted labels are not in the ground-truth, they are essentially related in semantic sense. For example, the ``swimmers" in image (b) would be much natural when it comes to the ``pool". Moreover, we can also see that PrML would make supplementary predictions to describe images, enriching the corresponding content. For example, the ``grass" in image (c) and ``buildings" in image (d) are the missing objects in ground-truth labels.

\textbf{Validation of privileged label features.} We then validate the effectiveness of the proposed privileged information for multi-label learning. As discussed previously, the privileged label features serve as an guidance or comments from an Oracle teacher to connect the learning of all the labels. For the sake of fairness, we simply implement the validation experiments with LEML (without privileged label features) and PrML (with privileged label features). Note that our proposed privileged label features are composed with the values of labels; however, not all labels have prominent connections in multi-label learning \cite{sun2014multi}. Thus we selectively construct the privileged label features with respect to each label.

Particularly, we just use K-nearest neighbor rule to form the pool per label. For each label, only labels in its label pool, instead of the whole label set, are reckoned to provide mutual guidance during its learning. In our implementation, we simply utilize Hamming distance to accomplish search of K-nearest neighbor on the dataset corel5k. The experimental setting is the same with before and both algorithms share the same embedding dimension $k$. Moreover, we carry out independent tests ten times and the average results are shown in Figure \ref{fig1}.

As shown in Figure \ref{fig1}, we have the following two observations. (a) PrML is clearly superior to LEML when we select enough labels as privileged label features, \eg~more than 350 labels in corel5k dataset. Since their only difference lies in the usage of the privileged information, we can conclude that the guidance from the privileged information, \ie~the proposed privileged label features, can significantly improve the performance of multi-label learning. (b) With more labels involved in the privileged label features, the performance of PrML keeps improving in a steady speed, and when the dimension of privileged label features is large enough, the performance tends to stabilize on the whole.

The number of labels is directly related to the complexity of correcting function defined as a linear function. Thus few labels might induce the low function complexity, and the correcting function can not determine the optimal slack variables. In this way, the fault-tolerant capacity would be crippled and thus the performance is even worse than LEML. For example, when the dimension of privileged labels is less than 250 on corel5k, the Hamming loss, One-error, Coverage and Ranking loss of PrML is much larger than LEML. In contrast, overmuch labels might introduce unnecessary guidance of labels, and the \textit{extra} labels thus make no contribution to the further improvement of classification performance. For instance, the performance with 365 labels involved in privileged label features would be on par with that of all the other (373) labels in Hamming loss, One-error, Ranking loss and Average precision. Moreover, in real applications, it is still a safe choice that all other labels are involved in privileged information.

\begin{table*}[ht]
	\caption{Average predictive performance (mean $\pm$ std. deviation) of ten indepedent trails for various multi-label learning methods. In each trail, 50\% examples are randomly selected without repeating as training set and the rest as testing set. The top performance among all methods is marked in boldface.}
	\label{biao}
	\centering
	\scriptsize
	\begin{tabular}{cr|c|c|c|c|c|c}%
		dataset& method& Hamming loss $\downarrow$ & One-error $\downarrow$& Coverage $\downarrow$& Ranking loss $\downarrow$  & Aver precision $\uparrow$ &Mac AUC $\uparrow$    \\ \hline
		\multirow{4}{*}{enron}&BR &0.060$\pm$0.001&0.498$\pm$0.012&0.595$\pm$0.010&0.308$\pm$0.007&0.449$\pm$0.011&0.579$\pm$0.007\\
		&ECC  &0.056$\pm$0.001&0.293$\pm$0.008&0.349$\pm$0.014&0.133$\pm$0.004&0.651$\pm$0.006&0.646$\pm$0.008\\
		&RAKEL&0.058$\pm$0.001&0.412$\pm$0.016&0.523$\pm$0.008&0.241$\pm$0.005&0.539$\pm$0.006&0.596$\pm$0.007\\
		&LEML&\textbf{0.049$\pm$0.002}&0.320$\pm$0.004&0.276$\pm$0.005&0.117$\pm$0.006&0.661$\pm$0.004&0.625$\pm$0.007 \\
		&ML$^2$& 0.051$\pm$0.001&\textbf{0.258$\pm$0.090}&0.256$\pm$0.017&0.090$\pm$0.012&0.681$\pm$0.053&\textbf{0.714$\pm$0.021}\\ \hline
		&PrBR&0.053$\pm$0.001 & 0.342$\pm$0.010 & 0.238$\pm$0.006 & \textbf{0.088$\pm$0.003} & 0.618$\pm$0.004& 0.638$\pm$0.005\\
		&PrML&0.050$\pm$0.001&0.288$\pm$0.005&\textbf{0.221$\pm$0.005}&\textbf{0.088$\pm$0.006}&\textbf{0.685$\pm$0.005}&0.674$\pm$0.004 \\ \hline \hline
		\multirow{4}{*}{yeast} &BR &0.201$\pm$0.003&0.256$\pm$0.008&0.641$\pm$0.005&0.315$\pm$0.005&0.672$\pm$0.005&0.565$\pm$0.003\\
		&ECC  &0.207$\pm$0.003&0.244$\pm$0.009&0.464$\pm$0.005&0.186$\pm$0.003&0.752$\pm$0.006&0.646$\pm$0.003\\
		&RAKEL&0.202$\pm$0.003&0.251$\pm$0.008&0.558$\pm$0.006&0.245$\pm$0.004&0.720$\pm$0.005&0.614$\pm$0.003\\
		&LEML&0.201$\pm$0.004&0.224$\pm$0.003&0.480$\pm$0.005&0.174$\pm$0.004&0.751$\pm$0.006&0.642$\pm$0.004 \\
		&ML$^2$&\textbf{0.196$\pm$0.003}&0.228$\pm$0.009&\textbf{0.454$\pm$0.004}&0.168$\pm$0.003&0.765$\pm$0.005&\textbf{0.702$\pm$0.007} \\ \hline
		&PrBR&0.227$\pm$0.004&0.237$\pm$0.006&0.487$\pm$0.005&0.204$\pm$0.003&0.719$\pm$0.005&0.623$\pm$0.004 \\
		&PrML&0.201$\pm$0.003&\textbf{0.214$\pm$0.005}&0.459$\pm$0.004&\textbf{0.165$\pm$0.003}&\textbf{0.771$\pm$0.003}&0.685$\pm$0.003 \\ \hline \hline
		\multirow{4}{*}{corel5k} &BR&0.012$\pm$0.001&0.849$\pm$0.008 &0.898$\pm$0.003&0.655$\pm$0.004&0.101$\pm$0.003&0.518$\pm$0.001\\
		&ECC  &0.015$\pm$0.001&0.699$\pm$0.006&0.562$\pm$0.007&0.292$\pm$0.003&0.264$\pm$0.003&0.568$\pm$0.003\\
		&RAKEL&0.012$\pm$0.001&0.819$\pm$0.010&0.886$\pm$0.004&0.627$\pm$0.004&0.122$\pm$0.004&0.521$\pm$0.001\\
		&LEML&\textbf{0.010$\pm$0.001}&0.683$\pm$0.006&0.273$\pm$0.008&0.125$\pm$0.003&0.268$\pm$0.005&0.622$\pm$0.006 \\
		&ML$^2$& \textbf{0.010$\pm$0.001}&\textbf{0.647$\pm$0.007}&0.372$\pm$0.006&0.163$\pm$0.003&\textbf{0.297$\pm$0.002}&\textbf{0.667$\pm$0.007}\\ \hline
		&PrBR&\textbf{0.010$\pm$0.001}&0.740$\pm$0.007&0.367$\pm$0.005&0.165$\pm$0.004&0.227$\pm$0.004&0.560$\pm$0.005\\
		&PrML&\textbf{0.010$\pm$0.001}&0.675$\pm$0.003&\textbf{0.266$\pm$0.007}&\textbf{0.118$\pm$0.003}&0.282$\pm$0.005&0.651$\pm$0.004  \\ \hline \hline
		\multirow{4}{*}{bibtex} &BR&0.015$\pm$0.001&0.559$\pm$0.004&0.461$\pm$0.006&0.303$\pm$0.004&0.363$\pm$0.004&0.624$\pm$0.002 \\
		&ECC  &0.017$\pm$0.001&0.404$\pm$0.003&0.327$\pm$0.008&0.192$\pm$0.003&0.515$\pm$0.004&0.763$\pm$0.003\\
		&RAKEL&0.015$\pm$0.001&0.506$\pm$0.005&0.443$\pm$0.006&0.286$\pm$0.003&0.399$\pm$0.004&0.641$\pm$0.002\\
		&LEML&0.013$\pm$0.001&0.394$\pm$0.004&0.144$\pm$0.002&0.082$\pm$0.003&0.534$\pm$0.002&0.757$\pm$0.003 \\
		&ML$^2$&0.013$\pm$0.001&\textbf{0.365$\pm$0.004}&\textbf{0.128$\pm$0.003}&0.067$\pm$0.002&\textbf{0.596$\pm$0.004}&\textbf{0.911$\pm$0.002} \\ \hline
		&PrBR&0.014$\pm$0.001&0.426$\pm$0.004&0.178$\pm$0.010&0.096$\pm$0.005&0.529$\pm$0.009&0.702$\pm$0.003 \\
		&PrML&\textbf{0.012$\pm$0.001}&0.367$\pm$0.003&0.131$\pm$0.007&\textbf{0.066$\pm$0.003}&0.571$\pm$0.004&0.819$\pm$0.005 \\ \hline\hline
		\multirow{4}{*}{eurlex} &BR&0.018$\pm$0.004&0.537$\pm$0.002&0.322$\pm$0.008&0.186$\pm$0.009&0.388$\pm$0.005&0.689$\pm$0.007\\
		&ECC  &0.011$\pm$0.003&0.492$\pm$0.003&0.298$\pm$0.004&0.155$\pm$0.006&0.458$\pm$0.004&0.787$\pm$0.009\\
		&RAKEL&0.009$\pm$0.004&0.496$\pm$0.007&0.277$\pm$0.009&0.161$\pm$0.001&0.417$\pm$0.010&0.822$\pm$0.005\\
		&LEML& 0.003$\pm$0.002&0.447$\pm$0.005&0.233$\pm$0.003&0.103$\pm$0.010&0.488$\pm$0.006&0.821$\pm$0.014\\
		&ML$^2$&\textbf{0.001$\pm$0.001}&0.320$\pm$0.001&\textbf{0.171$\pm$0.003}&\textbf{0.045$\pm$0.007}&0.497$\pm$0.003&0.885$\pm$0.003 \\ \hline
		&PrBR&0.007$\pm$0.008&0.484$\pm$0.003&0.229$\pm$0.009&0.108$\pm$0.009&0.455$\pm$0.003&0.793$\pm$0.008\\
		&PrML&\textbf{0.001$\pm$0.002}&\textbf{0.299$\pm$0.003}&0.192$\pm$0.008&0.057$\pm$0.002&\textbf{0.526$\pm$0.009}&\textbf{0.892$\pm$0.004}\\ \hline\hline
		\multirow{4}{*}{mediamill} &BR&0.031$\pm$0.001&0.200$\pm$0.003& 0.575$\pm$0.003&0.230$\pm$0.001&0.502$\pm$0.002&0.510$\pm$0.001\\
		&ECC  &0.035$\pm$0.001&0.150$\pm$0.005&0.467$\pm$0.009&0.179$\pm$0.008&0.597$\pm$0.014&0.524$\pm$0.001\\
		&RAKEL&0.031$\pm$0.001&0.181$\pm$0.002&0.560$\pm$0.002&0.222$\pm$0.001&0.521$\pm$0.001&0.513$\pm$0.001\\
		&LEML&0.030$\pm$0.001&\textbf{0.126$\pm$0.003}&0.184$\pm$0.007&0.084$\pm$0.004&0.720$\pm$0.007&0.699$\pm$0.010 \\
		&ML$^2$& 0.035$\pm$0.002&0.231$\pm$0.004&0.278$\pm$0.003&0.121$\pm$0.003&0.647$\pm$0.002&\textbf{0.847$\pm$0.003}\\ \hline
		&PrBR&0.031$\pm$0.001&0.147$\pm$0.005&0.255$\pm$0.003&0.092$\pm$0.002& 0.648$\pm$0.003&0.641$\pm$0.004\\
		&PrML&\textbf{0.029$\pm$0.002}&0.130$\pm$0.002&\textbf{0.172$\pm$0.004}&\textbf{0.055$\pm$0.006}&\textbf{0.726$\pm$0.002}&0.727$\pm$0.008  \\ \hline
	\end{tabular}
\end{table*}

\subsection{Performance comparison}
Now we formally analyze the performance of the proposed privileged multi-label learning (PrML) in comparison with popular state-of-the-art methods. For each dataset, we randomly selected 50\% examples without repeating as the training set and the rest for testing. For the results' credibility, the dataset division process is implemented ten times independently and we recorded the corresponding results in each trail. Parameters $\gamma_1,\gamma_2$ and $C$ are determined in the same manner as before. As for the low embedding dimension $k$, following the wisdom of \cite{yu2014large}, we choose $k$ to be in $\{\lceil0.8L\rceil,\lceil0.85L\rceil,\lceil0.9L\rceil,\lceil0.95L\rceil\}$ and determined by cross validation using a part of training points. Particularly, we also cover the PrBR (privileged information + BR) to further investigate the proposed privileged information. The detailed results are reported in Table \ref{biao}.

From Table \ref{biao}, we can see the proposed PrML is comparable to the state-of-the-art ML$^2$ method, and significantly surpasses the other competing multi-label methods. Concretely, across all evaluation metrics and datasets, PrML ranks first in 52.8\% cases and the first two in all cases; even in the second place, PrML's performance is close to the top one. Comparing BR \& PrBR, and LEML \& PrML, we can safely infer that the privileged information plays an important role in enhancing the classification performance of multi-label predictors. Besides, in all the 36 cases, PrML wins 34 cases against PrBR and plays a tie twice in Ranking loss on enron and Hamming loss on corel5k respectively, which implies that the low-rank structure in PrML has positive impact in further improving the multi-label performance. Therefore, we can see PrML has inherited the merits of both low-rank parameter structure and privileged label information. In addition, PrML and LEML tend to perform better on datasets with more labels ($>$100). This might be because the low-rank assumption is more sensible when the number of labels is considerably large.

\section{Conclusion}
In this paper, we investigate the intrinsic privileged information to connect labels in multi-label learning. Tactfully, we regard the label values as the privileged label features. This strategy indicates that for each label's learning, other labels of each example may serve as its Oracle comments on the learning of this label. Without the requirement of additional data, we propose to actively construct privileged label features directly from the label space. Then we integrate the privileged information with the  low-rank hypotheses $Z=D^TW$ in multi-label learning, and formulate privileged multi-label learning (PrML) as a result. During the optimization, both the dictionary $D$ and the coefficient matrix $W$ can receive the comments from the privileged information. And experimental results show that with this very privileged information, the classification performance can be significantly improved. Thus we can also take the privileged label features as a way to boost the classification performance of the low-rank based models.

As for the future work, our proposed PrML can be easily extended into Kernel version to cohere with the nonlinearity in the  parameter space. Besides, using SVM-style $L_2$-hinge loss might further improve the training efficiency \cite{xu2016simple}. Theoretical guarantees will be also investigated.



\bibliographystyle{iclr2017_conference}
\bibliography{reference}
\newpage 
\begin{appendix}
	\section{Deduction of Eq.(6)'s dual problem Eq.(7)}
	Without the loss of generality, the objective of Eq.(6) can be rewritten into $\dfrac{1}{2}(\norm{\w_i}_2^2 + \gamma\norm{\tilde{\w}_i}_2^2)  + C\sum_{j=1}^n \langle \tilde{\w}_i, \tilde{\y}_{i,j}\rangle$ with $1\leftarrow \gamma_1, \gamma \leftarrow \gamma_2/\gamma_1, C\leftarrow C/\gamma_1$. Then its Lagrangian function is defined as
	\begin{equation*}
	\begin{split}
	&\mathcal{L}(\w_i,\tilde{\w}_i,\balpha,\bbeta)= \dfrac{1}{2}(\norm{\w_i}_2^2 + \gamma\norm{\tilde{\w}_i}_2^2)  + C\sum_{j=1}^n \langle \tilde{\w}_i, \tilde{\y}_{i,j}\rangle\\
	& -\sum_{j=1}^n \alpha_j[Y_{ij}\langle \w_i,D\x_j\rangle- 1+ \langle \tilde{\w}_i, \tilde{\y}_{i,j}\rangle] -\sum_{j=1}^n \beta_j \langle \tilde{\w}_i, \tilde{\y}_{i,j}\rangle\\
	= & \dfrac{1}{2}(\norm{\w_i}_2^2 + \gamma\norm{\tilde{\w}_i}_2^2) -\sum_{j=1}^n \alpha_jY_{ij}\langle \w_i,D\x_j\rangle\\
	&+ \sum_{j=1}^n (C-\alpha_j-\beta_j) \langle \tilde{\w}_i, \tilde{\y}_{i,j}\rangle+ \sum_{j=1}^n \alpha_j
	\end{split}
	\end{equation*}
	where $\alpha_j\geq0, \beta_j\geq0$ are the Lagrange multipliers. Setting the derivatives of $\mathcal{L}$ with respect to $\w_i$ and $\tilde{\w}_i$ to zero, we have
	\begin{equation*}
	\begin{split}
	\w_i &= \sum_{j=1}^n \alpha_jY_{ij}D\x_j,\\
	\tilde{\w}_i &= \frac{1}{\gamma}\sum_{j=1}^n (\alpha_j+\beta_j-C)\tilde{\y}_{i,j}.
	\end{split}
	\end{equation*}
	Then plugging them back to the Lagrangian function, we obtain
	\begin{equation*}
	\begin{split}
	\mathcal{L}(\balpha,\bbeta)=& -\frac{1}{2}\sum_{j=1}^n \sum_{q=1}^n \alpha_j\alpha_qY_{ij}Y_{iq}\langle D\x_j,D\x_q\rangle  + \sum_{j=1}^n \alpha_j\\
	-& \frac{1}{2\gamma}\sum_{j=1}^n \sum_{q=1}^n(\alpha_j+\beta_j-C)(\alpha_q+\beta_q-C)\langle\tilde{\y}_{i,j}, \tilde{\y}_{i,q} \rangle.
	\end{split}
	\end{equation*}
	Denote $\balpha=[\alpha_1;,,,;\alpha_n]$, $\bbeta=[\beta_1;,...;\beta_n]$ and $\y_i^* = [Y_{i1},Y_{i2},...,Y_{in}]^T$. Besides, let 
	$K_D\in\mathbb{R}^{n\times n}$ with $[j,q]$-th element being $K_D(j,q) = \langle D\x_j,D\x_q\rangle$, and $\tilde{K}_i\in\mathbb{R}^{n\times n}$ with $\tilde{K}_i(j,q) = \langle\tilde{\y}_{i,j}, \tilde{\y}_{i,q} \rangle$. Thus the dual problem of Eq.(6) is formulated as
	\begin{equation*}
	\begin{split}
	\max_{\balpha,\bbeta}~& -\frac{1}{2} (\balpha\circ\y_i^*)^TK_D(\balpha\circ\y_i^*)+\bm{1}^T\balpha  - \frac{1}{2\gamma}(\balpha+\bbeta-C\bm{1})^T\tilde{K}_i (\balpha+\bbeta-C\bm{1})
	\end{split}
	\end{equation*}
	with the constraints $\balpha\succeq0,\bbeta\succeq0$, \ie~$\alpha_j\geq0,\beta_j\geq0,\forall j\in [1:n],$ where $\circ$ is the Hadamard (element-wise) product of two vectors or matrices, and $\bm{1}$ is the vector with all ones.
	
	\section{Deduction of Eq.(10)'s dual problem Eq.(11)}	
	Similarly, the Lagrangian function of Problem (10) is defined as
	\begin{equation*}
	\begin{split}
	& \mathcal{L}(D,\tilde{W},A,B) = \dfrac{1}{2}\norm{D}_F^2 + \dfrac{\gamma_2}{2}\sum_{i=1}^L\norm{\tilde{\w}_i}_2^2  + C\sum_{i=1}^L\sum_{j=1}^n \langle \tilde{\w}_i, \tilde{\y}_{i,j}\rangle\\
	& -\sum_{i=1}^L\sum_{j=1}^n A_{ij}[Y_{ij}\langle D, \w_i\x_j^T\rangle - 1 + \langle \tilde{\w}_i, \tilde{\y}_{i,j}\rangle]\\ &-\sum_{i=1}^L\sum_{j=1}^n B_{ij} \langle \tilde{\w}_i, \tilde{\y}_{i,j}\rangle\\
	= &\dfrac{1}{2}\norm{D}_F^2 + \dfrac{\gamma_2}{2}\sum_{i=1}^L\norm{\tilde{\w}_i}_2^2 -\sum_{i=1}^L\sum_{j=1}^n A_{ij}Y_{ij}\langle D, \w_i\x_j^T\rangle\\
	& + \sum_{i=1}^L\sum_{j=1}^n A_{ij} + \sum_{i=1}^L\sum_{j=1}^n (C-A_{ij}-B_{ij})\langle \tilde{\w}_i, \tilde{\y}_{i,j}\rangle
	\end{split}
	\end{equation*}
	where $A_{ij}\geq0, B_{ij}\geq0$ are the Lagrange multipliers. Setting the derivatives of $\mathcal{L}$ with respect to $D$ and $\tilde{\w}_i$ to zero, we have
	\begin{equation*}
	\begin{split}
	D &= \sum_{i=1}^L\sum_{j=1}^n A_{ij}Y_{ij}\w_i\x_j^T,\\
	\tilde{\w}_i &= \frac{1}{\gamma_2}\sum_{j=1}^n (A_{ij}+B_{ij}-C)\tilde{\y}_{i,j}.
	\end{split}
	\end{equation*}	
	Then plugging them back to the Lagrangian function, we obtain
	\begin{equation*}
	\begin{split}
	&\mathcal{L}(A,B)=-\dfrac{1}{2} \sum_{i,p=1}^L\sum_{j,q=1}^n  A_{ij}Y_{ij} A_{pq}Y_{pq}
	\langle\w_i\x_j^T,\w_p\x_q^T\rangle \\
	&- \dfrac{1}{2\gamma_2}
	\sum_{i=1}^L\sum_{j,q=1}^n(A_{ij}+B_{ij}-C)(A_{iq}+B_{iq}-C)\langle\tilde{\y}_{i,j},\tilde{\y}_{i,q}\rangle + \sum_{i=1}^L\sum_{j=1}^n A_{ij} \\
	&=-\dfrac{1}{2} \sum_{i,p=1}^L\sum_{j,q=1}^n A_{ij}Y_{ij} A_{pq}Y_{pq}
	\langle\w_i\x_j^T,\w_p\x_q^T\rangle + \sum_{i=1}^L\sum_{j=1}^n A_{ij}\\
	&  - \dfrac{1}{2\gamma_2}\sum_{i,p=1}^L\sum_{j,q=1}^n(A_{ij}+B_{ij}-C)(A_{pq}+B_{pq}-C)P_{i,j,p,q}
	\end{split}
	\end{equation*}
	where $P_{i,j,p,q} = \langle\tilde{\y}_{i,j},\tilde{\y}_{p,q}\rangle \mathcal{I}(p=i)$ and $\mathcal{I}(\cdot)$ is the indicator function. To make the dual problem of Eq.(10) consistent with Eq.(7), we define a bijection $\phi: [1:L]\times[1:n] \rightarrow [1:Ln]$ as the row-based vectorization index mapping, \ie~$\phi([i,j]) = (i-1)n+j$. In a nutshell, we line the constraints (also the multipliers) according to the ``first label, then training points" principle, \ie~``first $i$, then $j$". Let $s = \phi([i,j])$, $t=\phi([p,q])$ and $vec(A)=\balpha, vec(B) = \bbeta, vec(Y) = \y^*$, where $vec(\cdot)$ means the row-based vectorization manipulation. Moreover, let $G\in\mathbb{R}^{L\times n}$ with $G_{ij} = \w_i\x_j^T$, and $K_W\in\mathbb{R}^{Ln\times Ln}$ with $K_W(s,t) =\langle G_{\phi^{-1}(s)}, G_{\phi^{-1}(t)}\rangle $. Then the dual objective function can be rewritten as
	\begin{equation*}
	\begin{split}
	\mathcal{L}(\balpha,\bbeta)=&-\dfrac{1}{2} \sum_{s=1}^{Ln}\sum_{t=1}^{Ln}  \alpha_t\alpha_sy^*_ty^*_s
	K_W(s,t) + \sum_{s=1}^{Ln} \alpha_s - \dfrac{1}{2\gamma_2}  \sum_{s=1}^{Ln}\sum_{t=1}^{Ln} (\alpha_s+\beta_s -C)(\alpha_t+\beta_t-C)\tilde{K}(s,t)\\
	= & -\frac{1}{2} (\balpha\circ\y^*)^TK_W(\balpha\circ\y^*)+\bm{1}^T\balpha- \frac{1}{2\gamma_2}(\balpha+\bbeta-C\bm{1})^T\tilde{K}(\balpha+\bbeta-C\bm{1})
	\end{split}
	\end{equation*}
	where $\tilde{K}=diag(\tilde{K}_1,\tilde{K}_2,...,\tilde{K}_L)$. And the KKT condition is then rewritten equivelently as
	\begin{equation*}
	\begin{split}
	D &= \sum_{i=1}^L\sum_{j=1}^n A_{ij}Y_{ij}\w_i\x_j^T  = \sum_{s=1}^{Ln}\alpha_sy^*_sG_{\phi^{-1}(s)}\\
	\tilde{\w}_i &= \frac{1}{\gamma_2}\sum_{j=1}^n (A_{ij}+B_{ij}-C)\tilde{\y}_{i,j} = \frac{1}{\gamma_2}\sum_{j=1}^n (\alpha_{\phi([i,j])} + \beta_{\phi([i,j])} - C) \tilde{\y}_{i,j}.
	\end{split}
	\end{equation*}

	\section{Proof of Theorem 1}
	First it is easy to know that Eq.(6) and Eq.(10) share the same optimization form with the following SVM+ problem:
	\begin{equation}\label{SVM+}
	\begin{split}
	\min_{\w,\tilde{\w}} & ~\dfrac{1}{2}(\norm{\w}_2^2 + \gamma\norm{\tilde{\w}}_2^2)  + C\sum_{j=1}^n \langle \tilde{\w}, \tilde{\x}_j\rangle\\
	\mbox{s.t.} & ~y_{j}\langle \w,\x_j\rangle\geq 1- \langle \tilde{\w}, \tilde{\x}_{j}\rangle \\
	&\langle \tilde{\w},\tilde{\x}_j\rangle\geq 0, \forall  j=1,...,n.
	\end{split}
	\end{equation}
	Let $F$ be the objective function:
	$F = \frac{1}{2}\norm{\w}_2^2 + \frac{\gamma}{2}\norm{\tilde{\w}}_2^2 + C\sum_{j=1}^n \langle\tilde{\w},\tilde{\x}_j\rangle$ and denote $\u: = (\w;\tilde{\w})$. Assume $\u_1$ and $\u_2$ are two solutions of Eq.(\ref{SVM+}), \ie~$F(\u_1)=F(\u_2)$ is the minimum of $F(\u)$ and let $\u_t = (1-t)\u_1 + t\u_2, t\in[0,1]$. Because the problem is convex, we have $F(\u_t)\leq(1-t)F(\u_1)+tF(\u_2) = F(\u_1)$, thus $F(\u_t)=F(\u_1)=F(\u_2)$ for all $t\in[0,1]$. As a result, the function $g(t) = F(\u_t)-F(\u_1)\equiv 0, t\in[0,1]$. Then we have
	\begin{align*}
	g(t)=&F(\u_t) - F(\u_1)\\
	=& \frac{1}{2}\norm{t(\w_2-\w_1)+\w_1}_2^2 + \frac{\gamma}{2}\norm{t(\tilde{\w}_2-\tilde{\w}_1) + \tilde{\w}_1}_2^2 \\
	&+ C\sum_{i=1}^n \langle t(\tilde{\w}_2-\tilde{\w}_1) + \tilde{\w}_1,\tilde{\x}_i\rangle - \frac{1}{2}\norm{\w_1}_2^2 - \frac{\gamma}{2}\norm{\tilde{\w}_1}_2^2 - C\sum_{i=1}^n \langle\tilde{\w}_1,\tilde{\x}_i\rangle\\
	= & \frac{1}{2} t^2 \norm{\w_2-\w_1}_2^2 + t\langle\w_2-\w_1,\w_1\rangle +  \frac{\gamma}{2} t^2 \norm{\tilde{\w}_2-\tilde{\w}_1}_2^2\\
	& + t\langle\tilde{\w}_2-\tilde{\w}_1,\tilde{\w}_1\rangle + tC\sum_{i=1}^n \langle \tilde{\w}_2-\tilde{\w}_1,\tilde{\x}_i\rangle
	\end{align*}
	Thus
	$$g''(t) = \norm{\w_2-\w_1}_2^2 + \gamma\norm{\tilde{\w}_2-\tilde{\w}_1}_2^2 = 0.$$
	For $\gamma>0$, then we have $\w_1=\w_2$ and $\tilde{\w}_1=\tilde{\w}_2$, which completes the proof.

\end{appendix}

\end{document}